\documentclass[10pt,twocolumn]{article}

\usepackage[utf8]{inputenc}
\usepackage[T1]{fontenc}
\usepackage{mathptmx}
\usepackage[margin=0.75in]{geometry}
\usepackage{microtype}
\usepackage{amsmath,amssymb,amsthm}
\usepackage{booktabs,multirow,array,makecell}
\usepackage{graphicx}
\usepackage{xcolor}
\usepackage{tikz}
\usetikzlibrary{positioning,arrows.meta,calc,fit,backgrounds}
\usepackage{enumitem}
\usepackage{caption}
\usepackage[round,authoryear]{natbib}
\definecolor{ink}{HTML}{1B211D}
\definecolor{muted}{HTML}{5A5E66}
\definecolor{pBlue}{HTML}{DCE5F7}   
\definecolor{pBlueD}{HTML}{3A5BA8}
\definecolor{pMint}{HTML}{D5F0E3}   
\definecolor{pMintD}{HTML}{1F7A55}
\definecolor{pSand}{HTML}{F4E8DB}   
\definecolor{pSandD}{HTML}{8A5A3A}
\definecolor{pRose}{HTML}{F7DDDD}   
\definecolor{pRoseD}{HTML}{B03A3A}
\definecolor{pPeach}{HTML}{FCE6D4}
\definecolor{pGrey}{HTML}{EEF0EC}
\usepackage{xurl}
\usepackage[colorlinks=true,linkcolor=pBlueD,citecolor=pBlueD,urlcolor=pBlueD]{hyperref}

\newtheorem{proposition}{Proposition}

\newcommand{\model}{LAVOIR}
\newcommand{\voi}{\mathrm{VOI}}
\newcommand{\E}{\mathbb{E}}
\newcommand{\pbar}[3]{\tikz[baseline=-0.6ex]{\fill[#1!25] (0,-0.08) rectangle (#3,0.08);\fill[#1] (0,-0.08) rectangle ({#2*#3},0.08);}}
\tikzset{
  box/.style={draw=ink, line width=0.6pt, rounded corners=1.5pt, inner sep=3pt, align=left, font=\scriptsize},
  lbl/.style={font=\tiny\ttfamily, text=muted},
  arr/.style={-{Stealth[length=4pt]}, line width=0.6pt, draw=ink},
  arrG/.style={-{Stealth[length=4pt]}, line width=0.9pt, draw=pMintD},
}
\setlist{itemsep=2pt,topsep=4pt}

\title{\textbf{\model: Teaching a Single-Pass Decision Encoder\\When and What to Ask with Amortized Value of Information}}

\author{
Furkan Y{\i}lmaz\\
\texttt{furkanyl509@gmail.com}
\and
Habibe Aleyna Ta\c{s}demir\\
\texttt{aleynattasdemir@gmail.com}
\and
Muhammed Faruk G\"{o}zay\\
\texttt{gozayfaruk@gmail.com}
}

\date{September 2026}

\begin{document}
\maketitle

\begin{abstract}
``System One'' decision models such as TypeSafe's Jev and its open counterpart Laya answer typed questions about a text in a single forward pass with calibrated probabilities, but they cannot ask for missing information: when a first message does not say what separates two departments, they guess. We present \model{} (\textbf{La}ya with \textbf{V}alue-\textbf{O}f-\textbf{I}nformation \textbf{R}outing), which places the candidate pieces of missing information (\emph{slots}) in the input next to the answer options, so that one forward pass returns both the decision distribution and, for every slot, the expected gain in the probability of the correct decision if the user were asked about it. VOI targets need no human labels: gold decisions come from schema rules, an LLM only verbalizes messages and answers, a model from another family checks every text, and pairing each message with several profiles makes regression on realized gains estimate the expected gain. A Gini-impurity cap bounds the predicted value by what a calibrated model can still gain. In a controlled study, decisions on seen schemas are statistically indistinguishable from the Bayes ceiling. The final model's question policy matches a greedy oracle VOI policy on seen schemas (AUC 0.799 vs.\ 0.797), and with at most 0.5 questions per conversation it is 14.1 points more accurate than never asking. On real ABCD conversations, one real exchange raises accuracy by 8.3 points where \model{} asks and leaves it unchanged where it does not; on SGD the cap lowers the asking rate from 93\% to 8.6\%. On Laya's twelve benchmarks \model{} is above Laya's reported scores on seven, and it answers a question in 31\,ms (median, GH200).
\end{abstract}

\section{Introduction}
\label{sec:intro}

Many production decisions over text are small, typed and frequent: which team should handle a ticket, whether an e-mail is phishing, which tool to call. LLMs can make them, but they generate text token by token and do not expose calibrated probabilities. TypeSafe introduced Jev \citep{typesafe2026jev} as the first \emph{System One} model, named after the fast mode of thinking in \citet{kahneman2011}: it returns typed decisions with probabilities from one parallel query, and its weights are not public. Laya \citep{laya2026}, released three days later as an open model with a Jev-compatible interface, places every answer option behind its own \texttt{[MASK]} marker in the input of a bidirectional encoder \citep{warner2024modernbert} and scores all options in one pass; the question schema is given at request time, so a new decision needs no retraining.

These models must decide from the text they are given, and real first messages are often underspecified. \emph{``My information was shared without permission, what can I do?''} may belong to the data-protection or to the legal team, depending on what was shared and what the customer wants. A human agent asks one question; a single-pass model guesses or escalates every uncertain case (Figure~\ref{fig:teaser}).

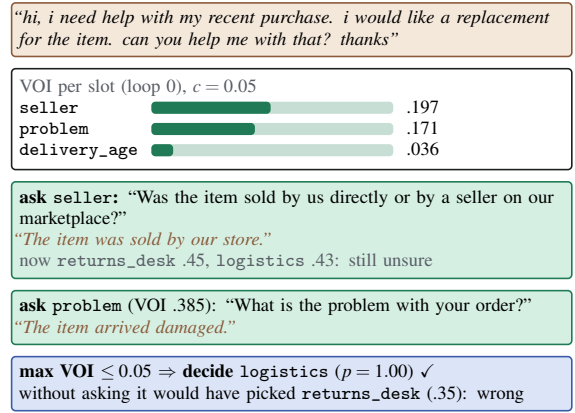
\begin{figure*}[t]
\centering
\begin{tikzpicture}[font=\scriptsize]
  \node[font=\small\bfseries, anchor=west] (tL) at (0,0) {Laya: decide or hand off};
  \node[box, fill=pSand, draw=pSandD, text width=7.2cm, below=3pt of tL.south west, anchor=north west] (mL)
    {\itshape ``hi, i need help with my recent purchase. i would like a replacement for the item. can you help me with that? thanks''};
  \node[box, fill=white, text width=7.2cm, below=4pt of mL] (pL) {%
    \begin{tabular}{@{}l@{\hspace{5pt}}l@{\hspace{5pt}}r@{}}
      \multicolumn{3}{@{}l}{\textcolor{muted}{$p(\text{team}\mid\text{message})$}}\\
      \texttt{refunds}             & \pbar{pBlueD}{0.357}{3.2cm} & .36 \\
      \texttt{returns\_desk}       & \pbar{pBlueD}{0.290}{3.2cm} & .29 \\
      \texttt{marketplace\_support}& \pbar{pBlueD}{0.220}{3.2cm} & .22 \\
      \texttt{warranty\_service}   & \pbar{pBlueD}{0.069}{3.2cm} & .07 \\
      \texttt{logistics}           & \pbar{pBlueD}{0.065}{3.2cm} & .06 \\
    \end{tabular}};
  \node[box, fill=pRose, draw=pRoseD, text width=7.2cm, below=4pt of pL] (hL)
    {\textbf{confidence $0.12 < 0.85$ $\Rightarrow$ hand off to a human}\\ best guess \texttt{refunds} (36\%): wrong};

  \begin{scope}[xshift=8.3cm]
  \node[font=\small\bfseries, anchor=west] (tR) at (0,0) {\model: ask, then decide};
  \node[box, fill=pSand, draw=pSandD, text width=7.2cm, below=3pt of tR.south west, anchor=north west] (mR)
    {\itshape ``hi, i need help with my recent purchase. i would like a replacement for the item. can you help me with that? thanks''};
  \node[box, fill=white, text width=7.2cm, below=4pt of mR] (vR) {%
    \begin{tabular}{@{}l@{\hspace{5pt}}l@{\hspace{5pt}}r@{}}
      \multicolumn{3}{@{}l}{\textcolor{muted}{VOI per slot (loop 0), $c=0.05$}}\\
      \texttt{seller}        & \pbar{pMintD}{0.493}{3.2cm} & .197 \\
      \texttt{problem}       & \pbar{pMintD}{0.428}{3.2cm} & .171 \\
      \texttt{delivery\_age} & \pbar{pMintD}{0.090}{3.2cm} & .036 \\
    \end{tabular}};
  \node[box, fill=pMint, draw=pMintD, text width=7.2cm, below=4pt of vR] (q1)
    {\textbf{ask \texttt{seller}:} ``Was the item sold by us directly or by a seller on our marketplace?''\\
     \textcolor{pSandD}{\itshape ``The item was sold by our store.''}\\
     \textcolor{muted}{now \texttt{returns\_desk} .45, \texttt{logistics} .43: still unsure}};
  \node[box, fill=pMint, draw=pMintD, text width=7.2cm, below=4pt of q1] (q2)
    {\textbf{ask \texttt{problem}} (VOI .385): ``What is the problem with your order?''\\
     \textcolor{pSandD}{\itshape ``The item arrived damaged.''}};
  \node[box, fill=pBlue, draw=pBlueD, text width=7.2cm, below=4pt of q2] (dR)
    {\textbf{max VOI $\le 0.05$ $\Rightarrow$ decide \texttt{logistics}} ($p=1.00$) $\checkmark$\\
     without asking it would have picked \texttt{returns\_desk} (.35): wrong};
  \end{scope}
\end{tikzpicture}
\caption{The same underspecified first message (a case of the \texttt{ecommerce\_returns} schema) given to Laya and to \model{}. Laya's confidence (one minus the normalized entropy, as Laya computes it) stays at 0.12, so its recommended usage (hand off below 0.85) sends the conversation to a human, and its best guess is wrong anyway; Laya never saw this schema in training. \model{} asks about the seller and then about the problem, skips the three slots whose answers would not change the decision, and routes to the correct team.}
\label{fig:teaser}
\end{figure*}

\model{} adds a second marker block for slots after the answer options. In the same forward pass that produces the decision distribution, a small \emph{VOI head} reads each slot marker and predicts how much asking about that slot would raise the probability of the correct decision; the model asks about the best slot if its value exceeds a question cost and otherwise decides or hands off (Figure~\ref{fig:loop}). The question \emph{content} is chosen by the model; its \emph{wording} is a template attached to the slot. The training signal needs no human labels: schema rules map complete user profiles to gold units, an LLM only verbalizes partial profiles, and because every message is paired with several profiles that differ in their hidden values, least-squares regression on the realized gain $p_k(y^\star)-p_0(y^\star)$ estimates the expected gain, an amortized value of information \citep{howard1966,rao2018}. Our contributions are:
\begin{itemize}
    \item \textbf{Amortized VOI for typed decisions}: a slot block, a zero-initialized segment embedding (with no slots the model reproduces Laya's logits bit for bit), a VOI head, and a Gini cap that bounds predicted value by the maximum expected gain of a calibrated model.
    \item \textbf{Label-free VOI targets} from rule-defined gold with an exact posterior, an answer simulator, and two-way cross-family checking, together with the schema design lessons that made generation work.
    \item \textbf{An evaluation against exact references}: the decisions match the Bayes ceiling and the question policy matches an oracle policy; on real conversations the questions go where answers help.
    \item \textbf{Findings about Laya's recipe}: its policy-gradient term brings no gain and slows the VOI head, and one temperature per question type miscalibrates heterogeneous mixtures.
\end{itemize}

\section{Related Work}
\label{sec:related}

\paragraph{System One models and decision encoders.}
Jev \citep{typesafe2026jev} returns typed answers (categories, scores, a choice among up to 255 options) with probabilities, decodes in parallel and is trained with what TypeSafe calls Reinforcement Learning for Calibrated Decisions (RLCD); it is API-only, with 236--276\,ms median latency in independent measurements \citep{jevbench_abdel,jevbench_nib}. Laya \citep{laya2026} implements this interface on ModernBERT-large \citep{warner2024modernbert} or mmBERT \citep{marone2025mmbert}, trains with strictly proper scoring rules \citep{gneiting2007} plus a group-baseline policy-gradient term \citep{williams1992,shao2024deepseekmath}, and calibrates with per-type temperatures \citep{guo2017}. GLiNER and GLiClass \citep{zaratiana2024gliner,stepanov2025gliclass} encode labels jointly with the text, and SCX Router \citep{stepanov2026scx} scores label tokens against a decoder cache for model routing \citep{ong2024routellm}. None of these models can ask for missing information.

\paragraph{Clarifying questions.}
\citet{rao2018} rank clarification questions by the expected value of perfect information, generating candidate answers. Qulac and ClariQ \citep{aliannejadi2019,aliannejadi2021} study clarification in search, \citet{yu2020interactive} ask binary questions chosen by information gain, and LLM-based work decides when to clarify \citep{kuhn2022clam,zhang2023clarify}, trains models to ask \citep{andukuri2024stargate}, or simulates answers to plan questions \citep{hu2024uot}. \model{} instead predicts the value of every candidate question in one encoder pass, without generating or simulating answers at inference.

\paragraph{Value of information, deferral and conformal prediction.}
VOI \citep{howard1966} and expected information \citep{lindley1956} are the classical criteria for choosing what to observe; deep adaptive design amortizes sequential design into a policy network \citep{foster2021dad}, and POMDP dialogue managers trade off asking and acting with simulated users \citep{williams2007pomdp,schatzmann2007}. Selective classification \citep{geifman2017} and learning to defer \citep{mozannar2020} decide whether to answer at all; KnowNo \citep{ren2023knowno} asks for help when a conformal set \citep{vovk2005,angelopoulos2021} has more than one option; our baseline B5 uses the same trigger with an uncalibrated prediction set. Because LLM-generated data carries exploitable artifacts \citep{li2023synthetic,gururangan2018,geirhos2020}, we keep labels out of the LLM and check texts with a model from another family.

\section{Problem Formulation}
\label{sec:formulation}

We follow Laya's interface: a \emph{state} $s$ (a message or JSON object) and a \emph{question} of type \textit{choice}, \textit{score} or \textit{noul} (yes/no) with options $o_1,\dots,o_m$; the model returns $p(\cdot\mid s,q)$.

\paragraph{Schemas and the exact posterior.}
A \emph{schema} has a set of units $\mathcal{U}$ (the options), decisive and non-decisive slots, a finite value set $V_k$ for every decisive slot, a conditional prior $P(v_k\mid v_{\pi(k)})$ with at most one parent slot, and an ordered rule list $g$ (the last rule is \emph{else}) that maps decisive-slot values to a unit. A \emph{profile} $\theta$ assigns every slot; its gold unit is $y^\star=g(\theta)$. What the model has seen about slot $k$ is an evidence set $E_k\subseteq V_k$: $\{\theta_k\}$ for a stated slot or a clean answer, $\{\theta_k,a\}$ for a \emph{partial} answer (``it might be $a$ or maybe $b$''), $V_k$ for ``I don't know''; an \emph{over-informative} answer also reveals one unasked slot. With a uniform likelihood inside each evidence set,
\begin{equation}
P(y\mid E)\propto\textstyle\sum_{\theta}P(\theta)\,\mathbb{1}[g(\theta)=y]\prod_k\mathbb{1}[\theta_k\in E_k],
\end{equation}
computed exactly by enumeration. The \emph{Bayes ceiling} of a test set is the mean of $\max_y P(y\mid E)$; no text-only model can exceed it in expectation.

\paragraph{Value of information.}
With $A_k(\cdot\mid\theta)$ the simulator's answer distribution for slot $k$, the \emph{oracle VOI} is the expected increase in the probability of the correct unit:
\begin{equation}
\begin{split}
\voi^\star(k\mid E)={}&\E_{\theta\sim P(\cdot\mid E)}\,\E_{a\sim A_k(\cdot\mid\theta)}\big[P(g(\theta)\mid E\oplus a)\big]\\
&-\textstyle\sum_{y}P(y\mid E)^2.
\end{split}
\label{eq:oracle}
\end{equation}
\begin{proposition}[Gini bound]
\label{prop:gini}
$\voi^\star(k\mid E)\le 1-\sum_y P(y\mid E)^2=G(P(\cdot\mid E))$, the Gini impurity \citep{breiman1984} of the posterior.
\end{proposition}
\begin{proof}
The first term of Eq.~\eqref{eq:oracle} is an expected probability, hence at most 1.
\end{proof}
For a calibrated model $p\approx P(\cdot\mid E)$, the value of any question therefore vanishes as the model becomes certain; \S\ref{sec:gini} shows that a learned head does not discover this on its own out of distribution.

\paragraph{Question policy.}
\label{sec:policy}
Given $p$ and predicted values $\hat v_k$ for the slots not yet asked, \model{} asks about $k^\star=\arg\max_k\hat v_k$ if $\hat v_{k^\star}>c$ and fewer than $Q$ questions have been asked, appends the question and answer to the state and re-encodes; otherwise it hands off if $1-\max_y p_y>c_h$ and decides $\arg\max_y p_y$ if not (Figure~\ref{fig:loop}). A slot is never asked twice.

\begin{figure*}[t]
\centering
\begin{tikzpicture}[font=\scriptsize,
  adim/.style={draw=ink, line width=0.6pt, rounded corners=1.5pt, minimum width=2.9cm, minimum height=0.95cm, align=center, fill=white},
  x=1cm, y=1cm]
  \node[adim] (msg) at (0,0)    {\textbf{Message}\\\textcolor{muted}{first message / answer added}};
  \node[adim] (mod) at (3.6,0)  {\textbf{Model}\\\textcolor{muted}{$p$ and $\hat v_k$ in one pass}};
  \node[adim, fill=pMint!60] (chk) at (7.2,0) {\textbf{$\max_k \hat v_k > c$ ?}\\\textcolor{muted}{e.g.\ \texttt{request} .366 $\to$ yes}};
  \node[adim, fill=pMint] (ask) at (10.8,0) {\textbf{Ask slot $k^\star$}\\\textcolor{muted}{fixed question}};
  \draw[arr] (msg) -- (mod);
  \draw[arr] (mod) -- (chk);
  \draw[arrG] (chk) -- node[above, font=\tiny\bfseries, text=pMintD] {yes} (ask);
  \draw[arrG] (ask.south) -- ++(0,-0.75) -| node[pos=0.35, fill=white, font=\tiny\ttfamily, inner sep=2pt] {send question $\to$ wait for answer $\to$ re-encode} (msg.south);

  \node[adim, fill=pBlue!70] (hc) at (7.2,-2.75) {\textbf{$1-\max_y p_y > c_h$ ?}\\\textcolor{muted}{still unsure?}};
  \node[adim, fill=pBlue] (dec) at (3.6,-2.75) {\textbf{Decide}\\\textcolor{muted}{$\arg\max_y p_y$}};
  \node[adim, fill=pRose] (ho) at (10.8,-2.75) {\textbf{Hand off}\\\textcolor{muted}{to a human}};
  \draw[arr, preaction={draw=white, line width=4pt}] (chk.south) -- node[right, font=\tiny\bfseries, text=pBlueD, pos=0.84] {no, or $Q$ questions asked} (hc.north);
  \draw[arr] (hc) -- node[above, font=\tiny\bfseries, text=pBlueD] {no} (dec);
  \draw[arr] (hc) -- node[above, font=\tiny\bfseries, text=pRoseD] {yes} (ho);

  \coordinate (lt) at ($(msg.north)+(0,0.3)$);
  \coordinate (lb) at ($(ask.south)+(0,-0.95)$);
  \coordinate (et) at ($(dec.north)+(0,0.3)$);
  \begin{scope}[on background layer]
    \node[fit=(lt)(msg)(ask)(lb), fill=pMint!25, draw=pMintD!60, rounded corners=2pt, inner sep=6pt] (lp) {};
    \node[fit=(et)(dec)(ho), fill=pBlue!35, draw=pBlueD!50, rounded corners=2pt, inner sep=6pt] (ep) {};
  \end{scope}
  \node[anchor=north west, font=\tiny\bfseries, text=pMintD, inner sep=3pt] at (lp.north west) {LOOP $\cdot$ RUNS INSIDE THE MODEL};
  \node[anchor=north west, font=\tiny\bfseries, text=pBlueD, inner sep=3pt] at (ep.north west) {EXIT};
\end{tikzpicture}
\caption{The asking loop of \S\ref{sec:policy}. Every forward pass returns the decision distribution $p$ and a value $\hat v_k$ for each slot not yet asked. While the best slot's value exceeds the cost $c$ and the budget $Q$ is not used up, its fixed question is sent and the answer is appended to the conversation, which is re-encoded. Otherwise the model decides, or hands the case to a human if it is still unsure. In the example of Figure~\ref{fig:teaser} the loop runs twice before exiting.}
\label{fig:loop}
\end{figure*}
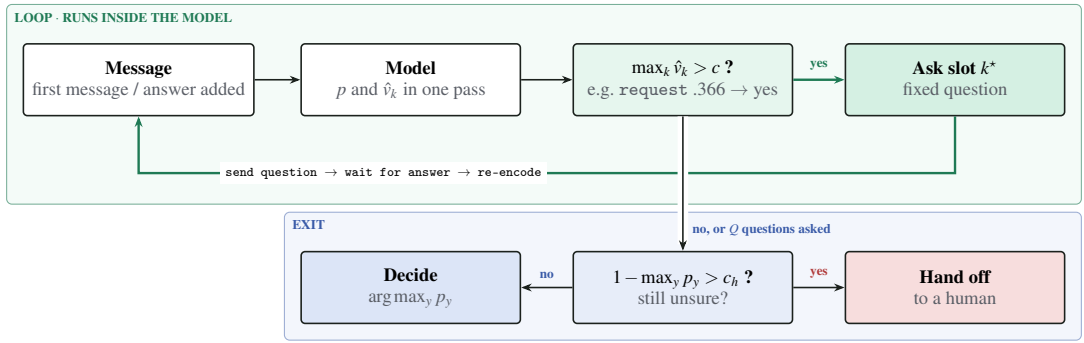

\section{Model}
\label{sec:model}

\paragraph{Input.} Laya's sequence has \texttt{[CLS]}, the question type and instruction, one \texttt{[MASK]} per option, and the state. We append an optional slot block introduced by the plain-text header ``missing information:'' (Figure~\ref{fig:seq}). Option and slot orders are re-shuffled every time an example is read, and the state is always a list of turns, so the target snapshot and the deployed model see the same format.

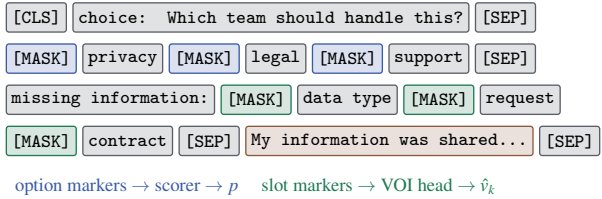
\begin{figure}[t]
\centering
\begin{tikzpicture}[font=\scriptsize\ttfamily,
  tok/.style={draw=#1, fill=#1!18, line width=0.5pt, rounded corners=1pt, inner sep=2pt, minimum height=11pt},
  tok/.default=muted, x=1pt, y=1pt]
  \node[tok] (a1) at (0,0) {[CLS]};
  \node[tok, right=2pt of a1] (a2) {choice: Which team should handle this?};
  \node[tok, right=2pt of a2] (a3) {[SEP]};
  \node[tok=pBlueD, below=4pt of a1.south west, anchor=north west] (b1) {[MASK]};
  \node[tok, right=2pt of b1] (b2) {privacy};
  \node[tok=pBlueD, right=2pt of b2] (b3) {[MASK]};
  \node[tok, right=2pt of b3] (b4) {legal};
  \node[tok=pBlueD, right=2pt of b4] (b5) {[MASK]};
  \node[tok, right=2pt of b5] (b6) {support};
  \node[tok, right=2pt of b6] (b7) {[SEP]};
  \node[tok, below=4pt of b1.south west, anchor=north west] (c0) {missing information:};
  \node[tok=pMintD, right=2pt of c0] (c1) {[MASK]};
  \node[tok, right=2pt of c1] (c2) {data type};
  \node[tok=pMintD, right=2pt of c2] (c3) {[MASK]};
  \node[tok, right=2pt of c3] (c4) {request};
  \node[tok=pMintD, below=4pt of c0.south west, anchor=north west] (d1) {[MASK]};
  \node[tok, right=2pt of d1] (d2) {contract};
  \node[tok, right=2pt of d2] (d3) {[SEP]};
  \node[tok=pSandD, right=2pt of d3] (d4) {My information was shared\dots};
  \node[tok, right=2pt of d4] (d5) {[SEP]};
  \node[font=\scriptsize\rmfamily, below=5pt of d1.south west, anchor=north west] {\textcolor{pBlueD}{option markers $\to$ scorer $\to p$} \quad \textcolor{pMintD}{slot markers $\to$ VOI head $\to \hat v_k$}};
\end{tikzpicture}
\caption{Input sequence. The second marker block lists the slots that could be asked; option markers (blue) feed the option scorer and slot markers (green) feed the VOI head in the same forward pass. The state (sand) is the conversation so far.}
\label{fig:seq}
\end{figure}

\paragraph{Architecture.} We keep Laya's decision head on ModernBERT-large (a question-type embedding, two transformer layers and an option scorer; \citealp{laya2026}) and add three components. A \emph{segment embedding} of the block id is initialized to zero, so with an empty slot block the logits are bit-identical to Laya's (\texttt{torch.equal}, real weights, fp32). The \emph{VOI head} computes $h_k=\mathrm{MLP}([z_k;f(p)])$ (LayerNorm over $d+4$ inputs, 256 GELU units, scalar output; 0.27M parameters), where $z_k$ is the head output at slot marker $k$ and $f(p)$ holds the largest probability, the top-two margin, the entropy divided by $\log m$, and $m/255$ (the features of Laya's act head), computed on the detached distribution. The \emph{Gini cap} gives
\begin{equation}
\hat v_k = G\big(\mathrm{softmax}(z/T)\big)\cdot\sigma(h_k),\qquad G(p)=1-\textstyle\sum_y p_y^2,
\label{eq:cap}
\end{equation}
with $G$ detached, so a confident model predicts a structurally small value. Uncapped variants output $h_k$ directly.

\paragraph{Losses.} Laya's decision loss is a soft cross-entropy plus a policy-gradient term that scores noisy logits with a proper reward \citep{gneiting2007,epstein1969} and applies REINFORCE with a group-mean baseline; we expose its weight $w_{\mathrm{RL}}$ ($1$ is Laya's recipe). The VOI target of slot $k$ is
\begin{equation}
t_k = \hat p_k(y^\star)-\hat p_0(y^\star),
\label{eq:target}
\end{equation}
where $\hat p_0$ is a frozen, temperature-scaled snapshot on the example and $\hat p_k$ the same after a sampled answer to slot $k$ (a \emph{probe}) is appended. We minimize the MSE between $\hat v_k$ and $t_k$, both divided by the target standard deviation, with weight $\lambda=1$. The target uses the gold unit, not the model's own risk: a self-referential target such as the drop in $1-\max_y p_y$ rewards any answer that makes the model confident, including a misleading one. It is also bounded in $[-1,1]$, unlike a log-loss difference. Because each message has several profiles, the regression estimates the snapshot's version of Eq.~\eqref{eq:oracle}.

\paragraph{Training.} A \emph{warm} phase trains the decision model on all examples (500 frozen-encoder steps, then four epochs; learning rates $2.5\times10^{-5}$ encoder and $10^{-4}$ head, cosine, bf16, global batch 64). Temperatures are fitted with L-BFGS \citep{liu1989lbfgs} on a 5\% held-out slice. The warm snapshot computes targets, two \emph{joint} epochs train $\mathcal{L}_{\mathrm{dec}}+\lambda\mathcal{L}_{\voi}$ with encoder learning rate $5\times10^{-6}$, and targets are recomputed from that snapshot for one more joint epoch, so the head regresses on values of a decision model that has stopped moving.

\section{Data}
\label{sec:data}

\paragraph{Schemas.} We wrote twelve customer-service schemas, eight for training and four held out for zero-shot evaluation, each with five or six units, three or four decisive slots and one non-decisive slot (Appendix~\ref{app:schemas}). Before generating any text, automatic gates checked reachability, unit masses, the share of units that need two or more slots, and the mean maximum posterior and VOI share of a dry run. Three pilots and one discarded full run showed that data quality is governed by the \emph{semantic independence} of slots, not by rule complexity: overlapping slots (``an unrecognized transaction'' and ``did you authorize it'') leak each other; a \emph{not applicable} value is indistinguishable from ``no'' in text; default-like values (``technical help'') are leaked by any vague message; non-decisive dates or amounts imply decisive facts, so the final non-decisive slot is the customer's name; and the generator must be forbidden to invent specifics when the topic is hidden.

\paragraph{Cases and answers.} A \emph{case} is a partial profile (the slots stated in the first message) with two to four complete profiles consistent with it, which differ in their hidden values and, typically, in their gold unit; without this pairing the regression would learn a deterministic mapping rather than an expectation. Each profile yields a message-only example, with probability 0.7 one with a question--answer pair, and with probability 0.4 one with two, plus a probe for every slot. Answers are clean (0.675), partial (0.125), ``I don't know'' (0.10) or over-informative (0.10); probes of already known slots are always clean.

\paragraph{Generation and two-way checking.} Messages (ten styles, at most 120 words) and answers are written by Qwen3.6-35B-A3B \citep{qwen3,qwen36} with vLLM \citep{kwon2023vllm}. The prompt lists the facts to state plainly and the topics not to mention, without their values. Gemma-4-26B-A4B \citep{gemma4} returns, for every decisive slot, a constrained guess and an evidence level; a text is regenerated (up to three times) if a hidden slot is guessed with explicit evidence (\emph{leakage}) or a stated slot is not recovered (\emph{fidelity}). The final run (two GH200 nodes, 12\,min) produced 21{,}601 training examples (8$\times$500 cases), 5{,}379 seen-schema test examples and 5{,}432 zero-shot test examples. First-attempt message rejection was 1--10\% for seven training schemas and 22\% for privacy, whose ``type of data'' and ``request'' slots overlap irreducibly; answer rejection was at most 2.8\%, and partial answers were guessed correctly 46--62\% of the time (expected 50\%).

\paragraph{General single-turn data.} To keep general ability we add single-turn data in Laya's format (details in Appendix~\ref{app:data}): typed-decisions, CLINC150, MultiNLI, Yelp and SGD first turns \citep{laya2026,larson2019clinc,williams2018mnli,zhang2015charcnn,rastogi2020sgd} (19{,}000); tool choice from Nemotron post-training data \citep{nemotron2025data} (6{,}000, of which 5{,}447 train); the six sources Laya trained on with question-form yes/no items \citep{zhang2015charcnn,clark2019boolq,metsis2006enron,phishingdata,bajaj2016msmarco,ticketsdata} (19{,}437), with all 4{,}447 benchmark texts removed; and 27 open sources (38{,}200, train splits only). \emph{Controlled} models use the schema data, the first 19{,}000 and tool choice (46{,}048 examples); \emph{broad} models use everything (103{,}685).

\paragraph{Real conversations.} For out-of-distribution tests we use 918 ABCD test dialogues \citep{chen2021abcd} with ten flows as units, the first customer block as state and the agent's first reply plus the customer's second block as the real exchange, and 3{,}812 SGD first user turns \citep{rastogi2020sgd} from 20 services (2{,}766 from services unseen in training) with the service's arguments as slots, which cannot change the intent.

\section{Experimental Setup}
\label{sec:setup}

\begin{table*}[t]
\centering
\caption{Controlled study on seen and zero-shot (zs) schemas. \emph{Gap}: accuracy minus Bayes ceiling (95\% CI). $\rho$, top-1: predicted vs.\ oracle VOI. AUC: area under the accuracy--questions curve (0--2 questions). b0.5: best accuracy with $\le$0.5 questions per conversation. Upper block: group temperatures; lower: one temperature per type (Laya).}
\label{tab:main}
\resizebox{\textwidth}{!}{%
\begin{tabular}{lll ccc cc cc ccccc ccc}
\toprule
& & & \multicolumn{3}{c}{Single turn} & & & \multicolumn{2}{c}{VOI vs.\ oracle} & \multicolumn{5}{c}{AUC} & \multicolumn{3}{c}{b0.5} \\
\cmidrule(lr){4-6}\cmidrule(lr){9-10}\cmidrule(lr){11-15}\cmidrule(lr){16-18}
T & Arm & Split & Acc & Ceiling & Gap [95\% CI] & ECE & KL & $\rho$ & top-1 & B2 & B3 & B5 & \textbf{VOI} & ORACLE & B2 & B3 & \textbf{VOI} \\
\midrule
group & RL1 & seen & .653 & .659 & $-$.006 [$-$.016, .004] & .015 & .016 & .821 & .876 & .612 & .766 & .791 & \textbf{.800} & .803 & .612 & .669 & \textbf{.741} \\
group & RL0 & seen & .652 & .659 & $-$.006 [$-$.017, .004] & .013 & .012 & .831 & .923 & .611 & .763 & .791 & \textbf{.799} & .800 & .611 & .660 & \textbf{.755} \\
group & RL1 & zs & .517 & .667 & $-$.150 [$-$.162, $-$.137] & .130 & .993 & .690 & .558 & .502 & .590 & .607 & \textbf{.629} & .658 & .502 & .566 & \textbf{.612} \\
group & RL0 & zs & .512 & .667 & $-$.154 [$-$.166, $-$.142] & .099 & .747 & .730 & .745 & .492 & .574 & .595 & \textbf{.616} & .633 & .492 & .542 & \textbf{.579} \\
\midrule
single & RL1 & seen & .652 & .659 & $-$.007 & .078 & .063 & .817 & .873 & .614 & .753 & .772 & .803 & .802 & .614 & .614 & .733 \\
single & RL0 & seen & .648 & .659 & $-$.010 & .064 & .043 & .773 & .934 & .608 & .751 & .772 & .797 & .798 & .608 & .608 & .762 \\
single & RL1 & zs & .514 & .667 & $-$.152 & .071 & .492 & .682 & .554 & .499 & .581 & .580 & .622 & .653 & .499 & .499 & .609 \\
single & RL0 & zs & .522 & .667 & $-$.145 & .070 & .481 & .702 & .649 & .498 & .578 & .576 & .624 & .647 & .498 & .498 & .549 \\
\bottomrule
\end{tabular}}
\end{table*}

\paragraph{Models.} All models use ModernBERT-large and a single seed: \emph{controlled} models with $w_{\mathrm{RL}}\in\{0,1\}$ (RL0/RL1) and either one temperature per question type (Laya's procedure) or per source group; the \emph{broad} model ($w_{\mathrm{RL}}=0$, per-source temperatures, uncapped); and the final \model{}, the broad warm model with the joint phases re-run with the Gini cap. Training took 28--48\,min on 8--16 GH200 GPUs of the JUPITER booster. Laya and Jev numbers are those reported by Laya (Jev's are third-party measurements); apart from Figure~\ref{fig:teaser} we did not run Laya's checkpoints.

\paragraph{Policies.} All policies share the decision head and allow at most two questions: \textbf{B2} never asks; \textbf{B3} asks about a random unresolved slot when $\max_y p_y<\tau$; \textbf{B5} asks about the highest-VOI slot when the prediction set, the smallest set of units whose probabilities sum to at least $1-\alpha$, contains more than one unit, with $\alpha\in\{0.5,0.3,0.2,0.1,0.05,0.02,0.01\}$ swept like the other thresholds (the set is not calibrated on held-out data, so it carries no coverage guarantee); \textbf{VOI} is our policy; \textbf{ORACLE} asks about the slot with the largest exact $\voi^\star$. ORACLE is greedy (it looks one question ahead) and, like every policy, decides with the model's own decision head, so it is a strong reference rather than an upper bound, and its value differs between models. For every message-only test example we run one dialogue per profile, answering with the profile's probe, and sweep $c$, $\tau$ and $c_h$.

\paragraph{Metrics.} Accuracy, Bayes ceiling and their gap with a bootstrap 95\% interval \citep{efron1993}, ECE \citep{naeini2015} and KL to the posterior; Spearman $\rho$ and top-1 agreement between $\hat v$ and $\voi^\star$; \textbf{AUC}, the normalized area under the upper envelope of accuracy vs.\ mean questions (0--2); \textbf{b0.5}, the best accuracy with at most 0.5 questions per conversation; and the share of \emph{unnecessary} questions, those about a non-decisive slot or a slot whose value is already determined by the evidence, among all questions asked. AUC differences are given on the 0--1 scale, accuracy differences in percentage points.

\section{Results}
\label{sec:results}

\subsection{Controlled study}
\label{sec:controlled}

\paragraph{Decisions on seen schemas are Bayes-optimal.} Accuracy is 0.006 below the Bayes ceiling with an interval that contains zero (Table~\ref{tab:main}); a model exploiting leaked hints would exceed the ceiling, and one under-using the text would fall clearly below it. Two schemas looked suspicious in argmax accuracy (insurance 0.649 vs.\ a ceiling of 0.610), but comparing the probability of the gold unit with the posterior under a case-clustered bootstrap shows no excess on any schema (Appendix~\ref{app:results}); the argmax excess comes from nearly tied posteriors. Per-schema leakage tests should therefore use probabilities, not argmax accuracy.

\paragraph{The question policy matches the oracle.} With group temperatures the VOI policy reaches an AUC of 0.799--0.800 against 0.800--0.803 for ORACLE, with $\rho=0.82$--$0.83$ and 88--92\% top-1 agreement. With at most 0.5 questions per conversation it is 12.9--14.4 points more accurate than never asking (B2) and 7.2--9.5 points more accurate than asking at random when uncertain (B3). B5 uses the same VOI ranking but its AUC is 0.008--0.009 lower: deciding \emph{whether} to ask from the VOI itself beats deciding it from the size of a prediction set. At $c=0.02$ RL0 asks 1.09 questions per conversation for an accuracy of 0.889 with 1\% unnecessary questions (oracle: 1.07 for 0.890); adding $c_h=0.3$ hands off 18\% of conversations and is 0.973 accurate on the rest.

\paragraph{Zero-shot schemas.} On the four held-out schemas the VOI policy still beats every baseline (AUC 0.02--0.05 above B5; with group temperatures, b0.5 3.7--4.6 points above B3), but decisions are weak: 0.51--0.52 accuracy against a ceiling of 0.667. Trained on eight schemas, the model does not learn to read unseen rules from option descriptions.

\subsection{Real conversations and the Gini cap}
\label{sec:gini}

\begin{table}[t]
\centering
\caption{Broad model without and with the Gini cap (final \model{}). Question rates at $c=.02/.05/.1$. ABCD gains: accuracy change in points after the real first exchange in the cases where the policy (at $c=.02$) would and would not ask, with 95\% paired bootstrap intervals over cases (10{,}000 resamples; 529/389 cases for \model{}, 844/74 uncapped). ORACLE is greedy and uses each model's decision head, so it can lie below VOI.}
\label{tab:gini}
\resizebox{\columnwidth}{!}{%
\begin{tabular}{lcc}
\toprule
& Uncapped & \textbf{\model{}} \\
\midrule
SGD question rate & .93 / .87 / .69 & \textbf{.086 / .063 / .048} \\
SGD AUROC(VOI $\to$ error) & .70 & \textbf{.84} \\
SGD accuracy / ECE & .948 / .041 & .942 / .044 \\
ABCD question rate & .92 / .83 / .71 & \textbf{.58 / .49 / .41} \\
ABCD AUROC(VOI $\to$ error) & .65 & \textbf{.71} \\
ABCD accuracy / ECE & .653 / .222 & .657 / .186 \\
ABCD gain, asks & +2.7 [0.5, 5.0] & \textbf{+8.3 [4.9, 11.7]} \\
ABCD gain, does not ask & +2.7 [0.0, 6.8] & \textbf{$-$0.3 [$-$1.5, 1.0]} \\
Seen: AUC VOI / ORACLE & .795 / .797 & \textbf{.799} / .797 \\
Seen: b0.5 & .757 & .751 \\
Seen: $\rho$ / top-1 & .763 / .960 & \textbf{.851} / .937 \\
Zero-shot: $\rho$ / b0.5 & .662 / .535 & \textbf{.689 / .581} \\
\bottomrule
\end{tabular}}
\end{table}

\paragraph{Without the cap the model asks too much.} The broad model is 0.948 accurate on SGD and well calibrated (ECE 0.04), yet at $c=0.02$ it would ask in 93\% of SGD and 92\% of ABCD conversations. Among the 3{,}526 SGD cases with confidence $\ge0.99$ the median maximum VOI is 0.133, although by Proposition~\ref{prop:gini} the expected gain there is at most about 0.02. The head receives the decision distribution but has not learned the bound: out of distribution the slot representation dominates, and arguments that cannot change the intent get high value (amount 0.25, ride fare 0.18).

\paragraph{The cap fixes it structurally.} Re-running only the joint phases with Eq.~\eqref{eq:cap} leaves single-turn decisions essentially unchanged (calibration-slice accuracy 0.849 $\to$ 0.847, ABCD first-message accuracy 0.653 $\to$ 0.657) and drops the question rate on confident cases to zero (Table~\ref{tab:gini}). On SGD at $c=0.05$ the model asks in 40\% of its errors and 4\% of its correct decisions. On ABCD the uncapped model's questions were untargeted: the real exchange helped equally where it would and would not ask (+2.7 each). With the cap the whole gain falls where it asks (+8.3, 95\% CI [4.9, 11.7]) and none where it does not ($-$0.3 [$-$1.5, 1.0]); the difference between the groups is 8.6 [4.9, 12.2] points, against 0.0 [$-$4.7, 3.9] for the uncapped model, so the decision to ask now tracks the value of the answer. Over all conversations the capped model gains about 4.7 points from the exchange against 2.7 for the uncapped one: re-running the joint phases changed how the decision head reads a second turn, which our single-turn checks do not measure, so we compare the two groups within each model rather than the totals across models. On seen schemas the cap raises VOI quality ($\rho$ 0.76 $\to$ 0.85) and keeps the AUC at ORACLE level; ORACLE is slightly below VOI here (0.797 vs.\ 0.799) because it is greedy and shares the decision head. With at most 0.5 questions the final model reaches 0.751, against 0.610 for never asking and 0.663 for random questions. Zero-shot b0.5 improves by 4.6 points. What remains is calibration: on ABCD the model is wrong in 19\% of the cases where its confidence is at least 0.99, where the cap correctly stays silent but the confidence is wrong.

\subsection{External benchmarks}
\label{sec:bench}

\begin{table}[t]
\centering
\caption{Accuracy on Laya's benchmarks (seed 13, 400 cases per task). \emph{Ctrl}: controlled RL0. Bold: \model{} above Laya. Laya: best released checkpoint; Laya and Jev as reported by Laya (Jev's Banking77: 72 labels, $n=100$). $^\ast$Training split in our broad data and in Laya's; $^\dagger$related train splits only. Latencies come from different hardware (ours GH200, Laya T4, Jev through its API) and are not directly comparable.}
\label{tab:bench}
\resizebox{\columnwidth}{!}{%
\begin{tabular}{lcccc}
\toprule
Set & Ctrl & \textbf{\model{}} & Laya & Jev \\
\midrule
typed-decisions$^\ast$ & .785 & \textbf{.774} & .766 & .727 \\
Banking77 (77) & .575 & \textbf{.533} & .492 & .870 \\
MASSIVE-en (20 opt.) & .790 & \textbf{.805} & .783 & -- \\
jailbreak (ToxicChat) & .888 & \textbf{.825} & .762 & -- \\
model routing$^\dagger$ & .283 & \textbf{.754} & .659 & -- \\
toxicity (ToxicChat) & .525 & \textbf{.605} & .530 & -- \\
RAG relevance$^\ast$ & .503 & \textbf{.665} & .657 & -- \\
spam$^\ast$ & .480 & .993 & .993 & -- \\
phishing$^\ast$ & .625 & .978 & .993 & -- \\
AG News$^\ast$ & .765 & .905 & .953 & .910 \\
DAIR Emotion & .528 & .575 & .600 & .480 \\
support triage$^\ast$ & .285 & .383 & .522 & -- \\
\midrule
latency, p50 & 27\,ms & 31\,ms & 33--40\,ms & 236\,ms \\
\bottomrule
\end{tabular}}
\end{table}

We ran Laya's benchmark builder and metrics unchanged on AG News \citep{zhang2015charcnn}, DAIR Emotion \citep{saravia2018carer}, Banking77 \citep{casanueva2020banking}, support tickets \citep{ticketsdata}, Enron spam \citep{metsis2006enron}, phishing \citep{phishingdata}, ToxicChat \citep{lin2023toxicchat}, MS MARCO \citep{bajaj2016msmarco}, a model-routing task over GSM8K and MBPP items \citep{cobbe2021gsm8k,austin2021mbpp}, MASSIVE-en \citep{fitzgerald2023massive} and the typed-decisions test split. \model{} is above Laya on seven of twelve sets, tied on spam and below on four (Table~\ref{tab:bench}); Banking77, jailbreak and MASSIVE were never in our training data, and on typed-decisions it is also above Jev. Latencies were measured on different GPUs (ours GH200, Laya T4).

\paragraph{Yes/no collapse.} The controlled model beat Laya on typed-decisions, Banking77, MASSIVE and jailbreak but collapsed on every yes/no benchmark (``false'' on 398/400 spam cases, ``true'' on 399/400 RAG cases). All its yes/no training items were \emph{assertions} (``This trace requires human review.''), while benchmark instructions are \emph{questions} over dictionaries with named fields. Question-form yes/no items and the open sources removed the collapse (spam 0.480 $\to$ 0.993) and lifted model routing from 0.283 to 0.754, at a small cost in VOI rank correlation on seen schemas.

\subsection{Lessons about Laya's recipe}
\label{sec:rl}
\label{sec:temperature}

\paragraph{The policy-gradient term does not help.} In a miniature model, after the decision has converged the policy-gradient term keeps sending gradients to the shared parameters that are one to two orders of magnitude larger than the VOI gradient; with a fixed exploration scale its variance does not vanish at the optimum. With $w_{\mathrm{RL}}=0$ the VOI loss reached $10^{-4}$ in 200 steps, with $w_{\mathrm{RL}}=1$ it was still 0.06 after 1{,}000. At full scale the RL0 arm is ahead in every warm epoch, reaches the same decision accuracy (0.918 vs.\ 0.919), and learns the VOI head better (top-1 0.92 vs.\ 0.88 on seen and 0.75 vs.\ 0.56 on zero-shot schemas; Appendix~\ref{app:results}). Because Laya's reward is a differentiable function of the reported distribution, its gradient can be taken directly; policy gradients would be needed only if the reward depended on which questions were asked.

\paragraph{One temperature per type is not enough.} Laya fits one temperature per question type. On our mixture this gave $T=2.25$--$2.75$ for choice because the general data has one-hot targets, while our schemas, whose targets are soft posteriors, were already calibrated at $T=1$; the shared temperature made their KL ten times worse (0.004 $\to$ 0.04--0.06). Per-group temperatures ($T\approx1.02$ for our schemas) improve ECE on seen schemas five-fold, though not on zero-shot schemas (0.070--0.071 $\to$ 0.099--0.130; Table~\ref{tab:main}), yet the dialogue AUC barely moves: the VOI \emph{ranking} is robust to temperature, and calibration matters mainly for hand-off and, through the cap, for staying silent.

\section{Discussion and Conclusion}
\label{sec:future}
\model{} turns a single-pass decision encoder into one that knows when a question is worth asking and which one to ask, with the value of every candidate question predicted in the same forward pass as the decision. Rule-defined gold, several profiles per message, two-way cross-family checking and the Gini cap made this value learnable and trustworthy: on seen schemas the decisions are Bayes-optimal and the questions oracle-level, and on real conversations the questions go where answers help. The cap matters beyond our setting: an uncapped head can learn to output zero when the model is certain, and in distribution it does, but out of distribution it values fields that merely look important. Next steps are a Turkish model on MoganBERT-TR \citep{yilmaz2026moganbert}, more and more varied schemas to close the zero-shot gap, the ABCD training split for calibration on real dialogues, and a study with human participants.

\section*{Limitations}
\textbf{Synthetic dialogues:} answers in the controlled study come from a simulator whose kinds and rates we chose, and LLM-written text may be more cooperative than real users; the real-data experiments test single exchanges, not the full loop with people. \textbf{Zero-shot decisions} stay 15 points below the ceiling, and \textbf{calibration out of distribution} is weak (ABCD ECE 0.19). \textbf{Residual overlap} in the privacy schema caused 22\% message rejection, a mild selection effect. \textbf{Missing baselines:} a slow VOI that re-encodes every possible answer, an LLM agent, and an ablation of the self-referential target. Questions are slot \textbf{templates}; all runs use \textbf{one seed} and \textbf{English} only; Laya's numbers are taken from its repository, and the two models are trained on different mixtures.

\section*{Ethics Statement}
\model{} is a routing component: it chooses a unit, asks at most a few questions, or hands the case to a human. In sensitive domains (privacy, clinic routing, insurance) the hand-off threshold should be conservative and questions should be reviewed so that they do not request unnecessary personal information. All synthetic profiles are fictitious; the generator and checker (Qwen3.6-35B-A3B, Gemma-4-26B-A4B) are Apache-2.0 models that place no restriction on generated text. Most general sources are under permissive or share-alike licenses (Apache-2.0, MIT, CC0, CC BY, CC BY-SA), but several restrict use to non-commercial research: ANLI and the support tickets (CC BY-NC 4.0), MS MARCO, the Yelp reviews, MARC and QQP (their providers' terms); the phishing corpus is LGPL-3.0, and sources without a license on the Hugging Face Hub were used under their original distribution terms. We therefore release the weights under CC BY-NC 4.0 and do not redistribute any third-party data.

\section*{Code and Data Availability}
The code is available at \url{https://github.com/moganai/lavoir} under Apache-2.0: the model, sequence builder, losses, VOI target computation, question policy, evaluation, training and fine-tuning scripts, tests, and example workflow definitions with the data format. The final model with its fitted temperatures is available at \url{https://huggingface.co/moganai/lavoir} under CC BY-NC 4.0. The general single-turn mixture is not redistributed; it is built from the public sources listed in Appendix~\ref{app:data}.

\section*{Acknowledgments}
We acknowledge the EuroHPC Joint Undertaking for awarding this project access to the EuroHPC supercomputer JUPITER, hosted by the J\"ulich Supercomputing Centre (JSC), through the EuroHPC AI Factories Playground access call (project EHPC-AIF-2026PG01-1296). We thank Convai Innovations for releasing Laya's code, data and benchmark harness under an open license.


\appendix

\section{Schemas}
\label{app:schemas}

\begin{table*}[t]
\centering
\small
\caption{The twelve schemas. \emph{Profiles}: decisive-slot combinations with non-zero prior. \emph{Multi-slot}: prior mass of units that need at least two slots. \emph{Max post.}: mean maximum posterior of message-level evidence in the dry run. \emph{VOI$>$.05}: share of unknown slots with oracle VOI above 0.05.}
\label{tab:schemas}
\setlength{\tabcolsep}{4.5pt}
\begin{tabular}{llcccccc}
\toprule
Schema & Split & Units & Decisive/all & Profiles & Unit mass & Multi-slot & Max post. / VOI$>$.05 \\
\midrule
banking\_support & train & 5 & 4/5 & 44 & .100--.350 & 1.00 & .474 / .498 \\
ecommerce\_returns & train & 5 & 4/5 & 32 & .061--.324 & .70 & .455 / .474 \\
hr\_requests & train & 6 & 3/4 & 24 & .080--.240 & .80 & .405 / .532 \\
insurance\_claims & train & 5 & 4/5 & 32 & .109--.256 & 1.00 & .456 / .518 \\
it\_helpdesk & train & 5 & 4/5 & 80 & .130--.292 & .70 & .400 / .449 \\
privacy\_requests & train & 5 & 4/5 & 76 & .063--.440 & 1.00 & .549 / .492 \\
telecom\_support & train & 6 & 4/5 & 32 & .090--.270 & .80 & .379 / .410 \\
travel\_changes & train & 5 & 4/5 & 36 & .053--.350 & .85 & .516 / .504 \\
\midrule
saas\_support & zero-shot & 5 & 4/5 & 44 & .059--.452 & .88 & .562 / .468 \\
parcel\_delivery & zero-shot & 5 & 4/5 & 32 & .105--.341 & 1.00 & .485 / .492 \\
student\_affairs & zero-shot & 6 & 4/5 & 32 & .067--.282 & .80 & .441 / .439 \\
clinic\_routing & zero-shot & 6 & 4/5 & 32 & .050--.242 & .80 & .400 / .431 \\
\bottomrule
\end{tabular}
\end{table*}

\paragraph{Gates.} Every unit is reachable, unit masses lie in $[0.03,0.5]$, at least 5\% of the mass needs two or more slots, and a dry run of the exact posterior gives a mean maximum posterior in $[0.35,0.75]$ and a share of unknown slots with $\voi^\star>0.05$ in $[0.2,0.6]$. A validator enforces that the last rule is \emph{else} and that non-decisive slots occur in no rule; it caught a telecom slot used in no rule, fixed by adding a field-technician unit.

\paragraph{Pilot history.} In pilot 1 the partial-answer hint accuracy was 0.70--0.79 instead of 0.5; the bias was confined to probes of slots already stated in the message, which the checker read from context, so such probes are now always clean. Message leakage clustered on overlapping slot pairs (banking \emph{unrecognized transaction} / \emph{authorized = no}). Pilot 2 added \emph{not applicable} values, which the generator rendered as ``no'' and produced 1{,}293--1{,}404 answer contradictions per schema; they were removed. Pilot 3 cut message rejection to 0.32--0.35 in banking and privacy and answer rejection to 1.6--3.6\%. The first full run lost up to 36\% of cases in telecom and clinic\_routing because non-decisive dates implied decisive facts and the generator invented concrete problems when the topic was hidden; the final run replaced the non-decisive slot with the customer's name and forbade invented specifics.

\section{Training Data}
\label{app:data}

\begin{table*}[t]
\centering
\small
\setlength{\tabcolsep}{5pt}
\caption{Per-schema leakage test on seen schemas: model minus posterior probability of the gold unit, and accuracy minus ceiling, with case-clustered bootstrap 95\% intervals where the interval excludes or nearly excludes zero.}
\label{tab:leak}
\begin{tabular}{lcccc}
\toprule
& \multicolumn{2}{c}{RL0} & \multicolumn{2}{c}{RL1} \\
\cmidrule(lr){2-3}\cmidrule(lr){4-5}
Schema & $p(\text{gold})$ model$-$post. & Acc$-$ceiling & $p(\text{gold})$ model$-$post. & Acc$-$ceiling \\
\midrule
banking & $-$.020 [$-$.037, $-$.007] & $-$.054 [$-$.100, $-$.013] & $-$.014 [$-$.027, $-$.004] & $-$.038 [$-$.074, $-$.002] \\
ecommerce & +.001 [$-$.002, .003] & $-$.016 & $-$.000 & $-$.010 \\
hr & $-$.002 [$-$.005, .000] & $-$.000 & +.003 [.000, .005] & +.003 \\
insurance & $-$.005 [$-$.008, $-$.002] & +.041 [.002, .080] & $-$.000 [$-$.005, .004] & +.038 [$-$.001, .076] \\
it\_helpdesk & $-$.003 & $-$.006 & +.002 & $-$.008 \\
privacy & $-$.009 [$-$.015, $-$.005] & +.021 [$-$.015, .056] & $-$.004 & +.018 \\
telecom & $-$.004 & $-$.025 & $-$.001 & $-$.034 [$-$.067, $-$.001] \\
travel & $-$.002 & $-$.016 & $-$.003 & $-$.017 \\
\bottomrule
\end{tabular}
\end{table*}

\paragraph{General v1 (19{,}000).} Typed-decisions train split (6{,}000); CLINC150-plus with random subsets of 5--15 intents (5{,}000); MultiNLI entailment/contradiction as yes/no with the hypothesis as instruction (3{,}000); Yelp reviews as class-balanced score questions (2{,}000); SGD first user turns with the active intent among the intents of the service and two to four other services (3{,}000). Each source has 10--12 instruction wordings.

\paragraph{Tool choice (6{,}000).} One of 13 shards of Nemotron-Post-Training-Dataset-v1: the state is the user message, the options are the row's tool set with distractors, and the gold is the single tool called in the first assistant turn. We removed rows whose gold tool name appears in the user message (2{,}930, a dataset artifact), rows without a call or with several distinct calls, and set sizes outside 2--12; train and test (5{,}447 / 553) are split by tool set. LLM-written tool schemas were explored as VOI schemas but dropped: only 2 of 200 passed the gates, and an assistant does not ask the user which tool to call.

\paragraph{Laya sources (19{,}437).} AG News train (2{,}500), BoolQ (3{,}000), Enron spam train (1{,}937), phishing e-mails (2{,}000), MS MARCO v1.1 train (2{,}500) and support tickets (2{,}500), plus question-form yes/no items from CLINC, Yelp, MultiNLI and SGD (5{,}000). All 4{,}447 state texts of Laya's benchmark builder and 50 further Enron subject-line overlaps were removed.

\paragraph{Open diversity (38{,}200, 27 sources, train splits only).} Choice (17{,}000): DBpedia and Yahoo Answers \citep{zhang2015charcnn}, 20 Newsgroups \citep{lang1995}, Tweet Topic \citep{antypas2022tweettopic}, GoEmotions \citep{demszky2020goemotions}, TweetEval emotion and sentiment \citep{barbieri2020tweeteval}, MATH topic \citep{hendrycks2021math}, Dolly task type \citep{conover2023dolly}, and domain routing among GSM8K, MATH, MBPP, WritingPrompts, Natural Questions, SQL questions and Dolly \citep{cobbe2021gsm8k,austin2021mbpp,fan2018writingprompts,kwiatkowski2019nq,bmc2sql}. Yes/no (17{,}400): SST-2 \citep{socher2013sst}, IMDB \citep{maas2011imdb}, Civil Comments \citep{borkan2019civil}, Measuring Hate Speech \citep{kennedy2020hate}, TweetEval hate, offensive and irony, QQP \citep{iyer2017qqp}, PAWS \citep{zhang2019paws}, MRPC \citep{dolan2005mrpc}, SNLI \citep{bowman2015snli}, ANLI \citep{nie2020anli}, SciTail \citep{khot2018scitail}, CoLA \citep{warstadt2019cola}. Score (3{,}800): Amazon star ratings from the English part of MARC \citep{keung2020marc} (SetFit copy), toxicity level from the Civil Comments toxicity score, and hate level from the Measuring Hate Speech score. Multi-class sources use random option subsets, 30\% of negatively phrased questions have inverted targets, and states are often dictionaries with named fields. None of Laya's evaluation sets was used.

\section{Additional Results}
\label{app:results}

\paragraph{Leakage.} Table~\ref{tab:leak}: for insurance the model assigns the gold unit slightly \emph{less} probability than the posterior; the argmax excess comes from nearly tied posteriors (e.g.\ 0.44/0.56), and banking shows the same noise in the opposite direction. The examples on which the model is most confident relative to the posterior contain no hidden clue, only partial answers ambiguous between two values.

\paragraph{Policy gradient.} Table~\ref{tab:grad} shows the miniature model and Table~\ref{tab:warm} the full-scale controlled chains. At full scale the ratio between the policy-gradient and VOI gradient norms on the shared parameters was 1.5--7$\times$.

\begin{table}[!ht]
\centering
\small
\setlength{\tabcolsep}{4pt}
\caption{Miniature model ($d=128$, 100 examples, $w_{\mathrm{RL}}=1$): decision KL, VOI loss and gradient norms sent to the shared parameters.}
\label{tab:grad}
\begin{tabular}{rccccc}
\toprule
Step & KL & VOI MSE & $\lVert g_{\mathrm{RL}}\rVert$ & $\lVert g_{\mathrm{CE}}\rVert$ & $\lVert g_{\voi}\rVert$ \\
\midrule
0 & .511 & 1.730 & 0.017 & 0.001 & 11.50 \\
100 & .033 & 0.984 & 48.49 & 0.491 & 0.225 \\
300 & .011 & 0.673 & 28.71 & 0.158 & 8.384 \\
600 & .007 & 0.127 & 26.95 & 0.098 & 2.445 \\
\bottomrule
\end{tabular}
\end{table}

\begin{table}[!ht]
\centering
\small
\setlength{\tabcolsep}{4pt}
\caption{Full-scale controlled training on the calibration slice: warm phase at $T=1$, and joint phase (VOI Spearman and MSE against sampled targets; KL on our schemas at $T=1$ and after fitting one temperature per type).}
\label{tab:warm}
\begin{tabular}{lcc}
\toprule
& $w_{\mathrm{RL}}=1$ & $w_{\mathrm{RL}}=0$ \\
\midrule
warm ep.\ 1: acc / KL & .807 / .230 & .838 / .102 \\
warm ep.\ 2: acc / KL & .857 / .073 & .896 / .018 \\
warm ep.\ 3: acc / KL & .896 / .019 & .927 / .006 \\
joint: Spearman / MSE & .489 / .0108 & .527 / .0104 \\
joint: acc all / ours & .919 / .965 & .918 / .956 \\
joint: KL, $T{=}1\to$ fitted & .0061 $\to$ .0612 & .0040 $\to$ .0381 \\
\bottomrule
\end{tabular}
\end{table}

\paragraph{Temperatures.} Per-group temperatures give $T=1.015$--$1.017$ for our schemas and $T=3.78$--$4.34$ (choice), $5.52$--$7.51$ (score) and $2.04$--$4.69$ (yes/no) for the general data, whose KL drops from 0.70--0.94 to 0.28--0.31.

\paragraph{Effect of the broad data.} Broadening the data kept the dialogue behaviour on seen schemas (AUC 0.799 $\to$ 0.795, oracle 0.797; accuracy 0.652 $\to$ 0.648 against a ceiling of 0.659) but lowered VOI rank correlation (0.831 $\to$ 0.763) and zero-shot b0.5 (0.579 $\to$ 0.535), which the Gini cap restored to 0.581. Shuffling the options changes 1\% of AG News, 5\% of Emotion, 12\% of MASSIVE and 29\% of Banking77 decisions of the final model.

\section{Implementation}
\label{app:impl}
The experiments used an internal module on top of Laya's code (version 0.3.11, commit \texttt{1e28ac2}) that leaves Laya's files unchanged; the released \texttt{lavoir} package re-implements it as a standalone package with its own tests. The internal module's 149 tests cover bit-identical logits with no slots, marker positions under shuffling, the posterior against brute-force enumeration, zero VOI for known and non-decisive slots, $0\le\hat v_k\le G(p)$ under the cap, no VOI gradient in the option scorer, and the policy. Laya's act/escalate head is trained with a zero-weighted loss in its notebook, so we do not use it.


\begin{thebibliography}{99}
\footnotesize
\setlength{\itemsep}{0.5pt}

\bibitem[AbdelStark(2026)]{jevbench_abdel}
AbdelStark. 2026.
\newblock jev-benchmarks: Independent benchmarks of {TypeSafe} {Jev}.
\newblock GitHub repository, \url{https://github.com/AbdelStark/jev-benchmarks}. Accessed September 2026.

\bibitem[Almeida(2026)]{typesafe2026jev}
Diogo Almeida. 2026.
\newblock Introducing {System One} models and {Jev}.
\newblock TypeSafe AI blog, 15 September 2026, \url{https://typesafe.ai/blog/introducing-system-one-models-and-jev}; documentation: \url{https://docs.typesafe.ai/concepts/system-one}.

\bibitem[Aliannejadi et~al.(2021)]{aliannejadi2021}
Mohammad Aliannejadi, Julia Kiseleva, Aleksandr Chuklin, Jeff Dalton, and Mikhail Burtsev. 2021.
\newblock Building and evaluating open-domain dialogue corpora with clarifying questions.
\newblock In \emph{Proceedings of EMNLP 2021}.

\bibitem[Aliannejadi et~al.(2019)]{aliannejadi2019}
Mohammad Aliannejadi, Hamed Zamani, Fabio Crestani, and W.~Bruce Croft. 2019.
\newblock Asking clarifying questions in open-domain information-seeking conversations.
\newblock In \emph{Proceedings of SIGIR 2019}, pages 475--484.

\bibitem[Andukuri et~al.(2024)]{andukuri2024stargate}
Chinmaya Andukuri, Jan-Philipp Fr\"anken, Tobias Gerstenberg, and Noah~D. Goodman. 2024.
\newblock {STaR-GATE}: Teaching language models to ask clarifying questions.
\newblock In \emph{Conference on Language Modeling (COLM)}. arXiv:2403.19154.

\bibitem[Angelopoulos and Bates(2021)]{angelopoulos2021}
Anastasios~N. Angelopoulos and Stephen Bates. 2021.
\newblock A gentle introduction to conformal prediction and distribution-free uncertainty quantification.
\newblock arXiv:2107.07511.

\bibitem[Antypas et~al.(2022)]{antypas2022tweettopic}
Dimosthenis Antypas, Asahi Ushio, Jose Camacho-Collados, Vitor Silva, Leonardo Neves, and Francesco Barbieri. 2022.
\newblock Twitter topic classification.
\newblock In \emph{Proceedings of COLING 2022}.

\bibitem[Austin et~al.(2021)]{austin2021mbpp}
Jacob Austin, Augustus Odena, Maxwell Nye, Maarten Bosma, Henryk Michalewski, David Dohan, Ellen Jiang, Carrie Cai, Michael Terry, Quoc Le, and Charles Sutton. 2021.
\newblock Program synthesis with large language models.
\newblock arXiv:2108.07732.

\bibitem[b-mc2(2023)]{bmc2sql}
b-mc2. 2023.
\newblock sql-create-context.
\newblock Hugging Face dataset \texttt{b-mc2/sql-create-context} (CC BY 4.0), built from WikiSQL and Spider.

\bibitem[Bajaj et~al.(2016)]{bajaj2016msmarco}
Payal Bajaj, Daniel Campos, Nick Craswell, Li~Deng, Jianfeng Gao, Xiaodong Liu, Rangan Majumder, Andrew McNamara, Bhaskar Mitra, Tri Nguyen, et~al. 2016.
\newblock {MS MARCO}: A human generated machine reading comprehension dataset.
\newblock arXiv:1611.09268.

\bibitem[Barbieri et~al.(2020)]{barbieri2020tweeteval}
Francesco Barbieri, Jose Camacho-Collados, Luis Espinosa~Anke, and Leonardo Neves. 2020.
\newblock {TweetEval}: Unified benchmark and comparative evaluation for tweet classification.
\newblock In \emph{Findings of EMNLP 2020}.

\bibitem[Borkan et~al.(2019)]{borkan2019civil}
Daniel Borkan, Lucas Dixon, Jeffrey Sorensen, Nithum Thain, and Lucy Vasserman. 2019.
\newblock Nuanced metrics for measuring unintended bias with real data for text classification.
\newblock In \emph{Companion Proceedings of The Web Conference (WWW) 2019}, pages 491--500.

\bibitem[Bowman et~al.(2015)]{bowman2015snli}
Samuel~R. Bowman, Gabor Angeli, Christopher Potts, and Christopher~D. Manning. 2015.
\newblock A large annotated corpus for learning natural language inference.
\newblock In \emph{Proceedings of EMNLP 2015}, pages 632--642.

\bibitem[Breiman et~al.(1984)]{breiman1984}
Leo Breiman, Jerome~H. Friedman, Richard~A. Olshen, and Charles~J. Stone. 1984.
\newblock \emph{Classification and Regression Trees}.
\newblock Wadsworth.

\bibitem[Casanueva et~al.(2020)]{casanueva2020banking}
I\~nigo Casanueva, Tadas Tem\v{c}inas, Daniela Gerz, Matthew Henderson, and Ivan Vuli\'c. 2020.
\newblock Efficient intent detection with dual sentence encoders.
\newblock In \emph{Proceedings of the 2nd Workshop on NLP for Conversational AI}, pages 38--45.

\bibitem[Chen et~al.(2021)]{chen2021abcd}
Derek Chen, Howard Chen, Yi~Yang, Alexander Lin, and Zhou Yu. 2021.
\newblock Action-based conversations dataset: A corpus for building more in-depth task-oriented dialogue systems.
\newblock In \emph{Proceedings of NAACL-HLT 2021}.

\bibitem[Clark et~al.(2019)]{clark2019boolq}
Christopher Clark, Kenton Lee, Ming-Wei Chang, Tom Kwiatkowski, Michael Collins, and Kristina Toutanova. 2019.
\newblock {BoolQ}: Exploring the surprising difficulty of natural yes/no questions.
\newblock In \emph{Proceedings of NAACL-HLT 2019}, pages 2924--2936.

\bibitem[Cobbe et~al.(2021)]{cobbe2021gsm8k}
Karl Cobbe, Vineet Kosaraju, Mohammad Bavarian, Mark Chen, Heewoo Jun, Lukasz Kaiser, Matthias Plappert, Jerry Tworek, Jacob Hilton, Reiichiro Nakano, Christopher Hesse, and John Schulman. 2021.
\newblock Training verifiers to solve math word problems.
\newblock arXiv:2110.14168.

\bibitem[Conover et~al.(2023)]{conover2023dolly}
Mike Conover, Matt Hayes, Ankit Mathur, Jianwei Xie, Jun Wan, Sam Shah, Ali Ghodsi, Patrick Wendell, Matei Zaharia, and Reynold Xin. 2023.
\newblock Free {Dolly}: Introducing the world's first truly open instruction-tuned {LLM}.
\newblock Databricks blog, \url{https://www.databricks.com/blog/2023/04/12/dolly-first-open-commercially-viable-instruction-tuned-llm}.

\bibitem[{Convai Innovations}(2026)]{laya2026}
{Convai Innovations}. 2026.
\newblock Laya: Multilingual non-autoregressive {System~1} decision model.
\newblock Software and model card, version 0.3.11 (commit \texttt{1e28ac2}), \url{https://huggingface.co/convaiinnovations/laya}; code: \url{https://github.com/NandhaKishorM/laya}. Accessed September 2026.

\bibitem[Demszky et~al.(2020)]{demszky2020goemotions}
Dorottya Demszky, Dana Movshovitz-Attias, Jeongwoo Ko, Alan Cowen, Gaurav Nemade, and Sujith Ravi. 2020.
\newblock {GoEmotions}: A dataset of fine-grained emotions.
\newblock In \emph{Proceedings of ACL 2020}, pages 4040--4054.

\bibitem[Dolan and Brockett(2005)]{dolan2005mrpc}
William~B. Dolan and Chris Brockett. 2005.
\newblock Automatically constructing a corpus of sentential paraphrases.
\newblock In \emph{Proceedings of the Third International Workshop on Paraphrasing (IWP 2005)}.

\bibitem[Efron and Tibshirani(1993)]{efron1993}
Bradley Efron and Robert~J. Tibshirani. 1993.
\newblock \emph{An Introduction to the Bootstrap}.
\newblock Chapman \& Hall.

\bibitem[Epstein(1969)]{epstein1969}
Edward~S. Epstein. 1969.
\newblock A scoring system for probability forecasts of ranked categories.
\newblock \emph{Journal of Applied Meteorology}, 8(6):985--987.

\bibitem[Fan et~al.(2018)]{fan2018writingprompts}
Angela Fan, Mike Lewis, and Yann Dauphin. 2018.
\newblock Hierarchical neural story generation.
\newblock In \emph{Proceedings of ACL 2018}, pages 889--898.

\bibitem[FitzGerald et~al.(2023)]{fitzgerald2023massive}
Jack FitzGerald, Christopher Hench, Charith Peris, Scott Mackie, Kay Rottmann, Ana Sanchez, Aaron Nash, Liam Urbach, Vishesh Kakarala, Richa Singh, et~al. 2023.
\newblock {MASSIVE}: A 1{M}-example multilingual natural language understanding dataset with 51 typologically-diverse languages.
\newblock In \emph{Proceedings of ACL 2023}, pages 4277--4302.

\bibitem[Foster et~al.(2021)]{foster2021dad}
Adam Foster, Desi~R. Ivanova, Ilyas Malik, and Tom Rainforth. 2021.
\newblock Deep adaptive design: Amortizing sequential {B}ayesian experimental design.
\newblock In \emph{Proceedings of ICML 2021}, pages 3384--3395.

\bibitem[{Gemma Team}(2026)]{gemma4}
{Gemma Team, Google DeepMind}. 2026.
\newblock Gemma 4 technical report.
\newblock arXiv:2607.02770. Model used: \texttt{google/gemma-4-26B-A4B-it}.

\bibitem[Geifman and El-Yaniv(2017)]{geifman2017}
Yonatan Geifman and Ran El-Yaniv. 2017.
\newblock Selective classification for deep neural networks.
\newblock In \emph{Advances in Neural Information Processing Systems 30}, pages 4878--4887.

\bibitem[Geirhos et~al.(2020)]{geirhos2020}
Robert Geirhos, J\"orn-Henrik Jacobsen, Claudio Michaelis, Richard Zemel, Wieland Brendel, Matthias Bethge, and Felix~A. Wichmann. 2020.
\newblock Shortcut learning in deep neural networks.
\newblock \emph{Nature Machine Intelligence}, 2:665--673.

\bibitem[Gneiting and Raftery(2007)]{gneiting2007}
Tilmann Gneiting and Adrian~E. Raftery. 2007.
\newblock Strictly proper scoring rules, prediction, and estimation.
\newblock \emph{Journal of the American Statistical Association}, 102(477):359--378.

\bibitem[Guo et~al.(2017)]{guo2017}
Chuan Guo, Geoff Pleiss, Yu~Sun, and Kilian~Q. Weinberger. 2017.
\newblock On calibration of modern neural networks.
\newblock In \emph{Proceedings of ICML 2017}, pages 1321--1330.

\bibitem[Gururangan et~al.(2018)]{gururangan2018}
Suchin Gururangan, Swabha Swayamdipta, Omer Levy, Roy Schwartz, Samuel Bowman, and Noah~A. Smith. 2018.
\newblock Annotation artifacts in natural language inference data.
\newblock In \emph{Proceedings of NAACL-HLT 2018}, pages 107--112.

\bibitem[Hendrycks et~al.(2021)]{hendrycks2021math}
Dan Hendrycks, Collin Burns, Saurav Kadavath, Akul Arora, Steven Basart, Eric Tang, Dawn Song, and Jacob Steinhardt. 2021.
\newblock Measuring mathematical problem solving with the {MATH} dataset.
\newblock In \emph{NeurIPS 2021 Datasets and Benchmarks Track}.

\bibitem[Howard(1966)]{howard1966}
Ronald~A. Howard. 1966.
\newblock Information value theory.
\newblock \emph{IEEE Transactions on Systems Science and Cybernetics}, 2(1):22--26.

\bibitem[Hu et~al.(2024)]{hu2024uot}
Zhiyuan Hu, Chumin Liu, Xidong Feng, Yilun Zhao, See-Kiong Ng, Anh~Tuan Luu, Junxian He, Pang~Wei Koh, and Bryan Hooi. 2024.
\newblock Uncertainty of thoughts: Uncertainty-aware planning enhances information seeking in large language models.
\newblock In \emph{Advances in Neural Information Processing Systems 37}. arXiv:2402.03271.

\bibitem[Iyer et~al.(2017)]{iyer2017qqp}
Shankar Iyer, Nikhil Dandekar, and Korn\'el Csernai. 2017.
\newblock First {Quora} dataset release: Question pairs.
\newblock Quora blog.

\bibitem[Kahneman(2011)]{kahneman2011}
Daniel Kahneman. 2011.
\newblock \emph{Thinking, Fast and Slow}.
\newblock Farrar, Straus and Giroux.

\bibitem[Kennedy et~al.(2020)]{kennedy2020hate}
Chris~J. Kennedy, Geoff Bacon, Alexander Sahn, and Claudia von Vacano. 2020.
\newblock Constructing interval variables via faceted {R}asch measurement and multitask deep learning: A hate speech application.
\newblock arXiv:2009.10277. Dataset: \texttt{ucberkeley-dlab/measuring-hate-speech} (CC BY 4.0).

\bibitem[Keung et~al.(2020)]{keung2020marc}
Phillip Keung, Yichao Lu, Gy\"orgy Szarvas, and Noah~A. Smith. 2020.
\newblock The multilingual {A}mazon reviews corpus.
\newblock In \emph{Proceedings of EMNLP 2020}, pages 4563--4568. English part via \texttt{SetFit/amazon\_reviews\_multi\_en}; the original MARC license limits use to research.

\bibitem[Khot et~al.(2018)]{khot2018scitail}
Tushar Khot, Ashish Sabharwal, and Peter Clark. 2018.
\newblock {SciTail}: A textual entailment dataset from science question answering.
\newblock In \emph{Proceedings of AAAI 2018}.

\bibitem[Kuhn et~al.(2022)]{kuhn2022clam}
Lorenz Kuhn, Yarin Gal, and Sebastian Farquhar. 2022.
\newblock {CLAM}: Selective clarification for ambiguous questions with generative language models.
\newblock arXiv:2212.07769.

\bibitem[Kwiatkowski et~al.(2019)]{kwiatkowski2019nq}
Tom Kwiatkowski, Jennimaria Palomaki, Olivia Redfield, Michael Collins, Ankur Parikh, Chris Alberti, Danielle Epstein, Illia Polosukhin, Jacob Devlin, Kenton Lee, et~al. 2019.
\newblock Natural questions: A benchmark for question answering research.
\newblock \emph{Transactions of the Association for Computational Linguistics}, 7:452--466.

\bibitem[Kwon et~al.(2023)]{kwon2023vllm}
Woosuk Kwon, Zhuohan Li, Siyuan Zhuang, Ying Sheng, Lianmin Zheng, Cody~Hao Yu, Joseph~E. Gonzalez, Hao Zhang, and Ion Stoica. 2023.
\newblock Efficient memory management for large language model serving with {PagedAttention}.
\newblock In \emph{Proceedings of SOSP 2023}, pages 611--626.

\bibitem[Lang(1995)]{lang1995}
Ken Lang. 1995.
\newblock {NewsWeeder}: Learning to filter netnews.
\newblock In \emph{Proceedings of ICML 1995}, pages 331--339.

\bibitem[Larson et~al.(2019)]{larson2019clinc}
Stefan Larson, Anish Mahendran, Joseph~J. Peper, Christopher Clarke, Andrew Lee, Parker Hill, Jonathan~K. Kummerfeld, Kevin Leach, Michael~A. Laurenzano, Lingjia Tang, and Jason Mars. 2019.
\newblock An evaluation dataset for intent classification and out-of-scope prediction.
\newblock In \emph{Proceedings of EMNLP-IJCNLP 2019}, pages 1311--1316.

\bibitem[Li et~al.(2023)]{li2023synthetic}
Zhuoyan Li, Hangxiao Zhu, Zhuoran Lu, and Ming Yin. 2023.
\newblock Synthetic data generation with large language models for text classification: Potential and limitations.
\newblock In \emph{Proceedings of EMNLP 2023}.

\bibitem[Lin et~al.(2023)]{lin2023toxicchat}
Zi~Lin, Zihan Wang, Yongqi Tong, Yangkun Wang, Yuxin Guo, Yujia Wang, and Jingbo Shang. 2023.
\newblock {ToxicChat}: Unveiling hidden challenges of toxicity detection in real-world user-{AI} conversation.
\newblock In \emph{Findings of EMNLP 2023}.

\bibitem[Lindley(1956)]{lindley1956}
Dennis~V. Lindley. 1956.
\newblock On a measure of the information provided by an experiment.
\newblock \emph{The Annals of Mathematical Statistics}, 27(4):986--1005.

\bibitem[Liu and Nocedal(1989)]{liu1989lbfgs}
Dong~C. Liu and Jorge Nocedal. 1989.
\newblock On the limited memory {BFGS} method for large scale optimization.
\newblock \emph{Mathematical Programming}, 45:503--528.

\bibitem[{Liu}(2024)]{phishingdata}
Zefang Liu. 2024.
\newblock Phishing email dataset.
\newblock Hugging Face dataset \texttt{zefang-liu/phishing-email-dataset} (LGPL-3.0).

\bibitem[Maas et~al.(2011)]{maas2011imdb}
Andrew~L. Maas, Raymond~E. Daly, Peter~T. Pham, Dan Huang, Andrew~Y. Ng, and Christopher Potts. 2011.
\newblock Learning word vectors for sentiment analysis.
\newblock In \emph{Proceedings of ACL-HLT 2011}, pages 142--150.

\bibitem[Marone et~al.(2025)]{marone2025mmbert}
Marc Marone, Orion Weller, William Fleshman, Eugene Yang, Dawn Lawrie, and Benjamin Van~Durme. 2025.
\newblock {mmBERT}: A modern multilingual encoder with annealed language learning.
\newblock arXiv:2509.06888.

\bibitem[Metsis et~al.(2006)]{metsis2006enron}
Vangelis Metsis, Ion Androutsopoulos, and Georgios Paliouras. 2006.
\newblock Spam filtering with naive {B}ayes -- which naive {B}ayes?
\newblock In \emph{Third Conference on Email and Anti-Spam (CEAS 2006)}.

\bibitem[Mozannar and Sontag(2020)]{mozannar2020}
Hussein Mozannar and David Sontag. 2020.
\newblock Consistent estimators for learning to defer to an expert.
\newblock In \emph{Proceedings of ICML 2020}, pages 7076--7087.

\bibitem[Naeini et~al.(2015)]{naeini2015}
Mahdi Pakdaman Naeini, Gregory~F. Cooper, and Milos Hauskrecht. 2015.
\newblock Obtaining well calibrated probabilities using {B}ayesian binning.
\newblock In \emph{Proceedings of AAAI 2015}, pages 2901--2907.

\bibitem[Nie et~al.(2020)]{nie2020anli}
Yixin Nie, Adina Williams, Emily Dinan, Mohit Bansal, Jason Weston, and Douwe Kiela. 2020.
\newblock Adversarial {NLI}: A new benchmark for natural language understanding.
\newblock In \emph{Proceedings of ACL 2020}, pages 4885--4901.

\bibitem[{NVIDIA}(2025)]{nemotron2025data}
{NVIDIA}. 2025.
\newblock {Nemotron-Post-Training-Dataset-v1}.
\newblock Hugging Face dataset \texttt{nvidia/Nemotron-Post-Training-Dataset-v1}.

\bibitem[nibzard(2026)]{jevbench_nib}
nibzard. 2026.
\newblock decision-model-benchmark.
\newblock GitHub repository, \url{https://github.com/nibzard/decision-model-benchmark}. Accessed September 2026.

\bibitem[Ong et~al.(2024)]{ong2024routellm}
Isaac Ong, Amjad Almahairi, Vincent Wu, Wei-Lin Chiang, Tianhao Wu, Joseph~E. Gonzalez, M.~Waleed Kadous, and Ion Stoica. 2024.
\newblock {RouteLLM}: Learning to route {LLMs} with preference data.
\newblock arXiv:2406.18665.

\bibitem[{Qwen Team}(2026)]{qwen36}
{Qwen Team}. 2026.
\newblock {Qwen3.6-35B-A3B}.
\newblock Model card, \url{https://huggingface.co/Qwen/Qwen3.6-35B-A3B}. FP8 variant used.

\bibitem[Rao and Daum{\'e}~III(2018)]{rao2018}
Sudha Rao and Hal Daum\'e~III. 2018.
\newblock Learning to ask good questions: Ranking clarification questions using neural expected value of perfect information.
\newblock In \emph{Proceedings of ACL 2018}, pages 2737--2746.

\bibitem[Rastogi et~al.(2020)]{rastogi2020sgd}
Abhinav Rastogi, Xiaoxue Zang, Srinivas Sunkara, Raghav Gupta, and Pranav Khaitan. 2020.
\newblock Towards scalable multi-domain conversational agents: The schema-guided dialogue dataset.
\newblock In \emph{Proceedings of AAAI 2020}, pages 8689--8696.

\bibitem[Ren et~al.(2023)]{ren2023knowno}
Allen~Z. Ren, Anushri Dixit, Alexandra Bodrova, Sumeet Singh, Stephen Tu, Noah Brown, Peng Xu, Leila Takayama, Fei Xia, Jake Varley, Zhenjia Xu, Dorsa Sadigh, Andy Zeng, and Anirudha Majumdar. 2023.
\newblock Robots that ask for help: Uncertainty alignment for large language model planners.
\newblock In \emph{Conference on Robot Learning (CoRL)}. arXiv:2307.01928.

\bibitem[Saravia et~al.(2018)]{saravia2018carer}
Elvis Saravia, Hsien-Chi~Toby Liu, Yen-Hao Huang, Junlin Wu, and Yi-Shin Chen. 2018.
\newblock {CARER}: Contextualized affect representations for emotion recognition.
\newblock In \emph{Proceedings of EMNLP 2018}, pages 3687--3697.

\bibitem[Schatzmann et~al.(2007)]{schatzmann2007}
Jost Schatzmann, Blaise Thomson, Karl Weilhammer, Hui Ye, and Steve Young. 2007.
\newblock Agenda-based user simulation for bootstrapping a {POMDP} dialogue system.
\newblock In \emph{Proceedings of NAACL-HLT 2007, Companion Volume (Short Papers)}, pages 149--152.

\bibitem[Shao et~al.(2024)]{shao2024deepseekmath}
Zhihong Shao, Peiyi Wang, Qihao Zhu, Runxin Xu, Junxiao Song, Xiao Bi, Haowei Zhang, Mingchuan Zhang, Y.~K. Li, Y.~Wu, and Daya Guo. 2024.
\newblock {DeepSeekMath}: Pushing the limits of mathematical reasoning in open language models.
\newblock arXiv:2402.03300.

\bibitem[Socher et~al.(2013)]{socher2013sst}
Richard Socher, Alex Perelygin, Jean Wu, Jason Chuang, Christopher~D. Manning, Andrew Ng, and Christopher Potts. 2013.
\newblock Recursive deep models for semantic compositionality over a sentiment treebank.
\newblock In \emph{Proceedings of EMNLP 2013}, pages 1631--1642.

\bibitem[Stepanov et~al.(2026)]{stepanov2026scx}
Ihor Stepanov, Aleksandr Smechov, Mykhailo Shtopko, Dmytro Vodianytskyi, and Oleksandr Lukashov. 2026.
\newblock {SCX Router}: Streaming zero-shot model selection with a decoder-{KV} classifier and a real-world task ontology.
\newblock arXiv:2609.02292.

\bibitem[Stepanov et~al.(2025)]{stepanov2025gliclass}
Ihor Stepanov, Mykhailo Shtopko, Dmytro Vodianytskyi, Oleksandr Lukashov, Alexander Yavorskyi, and Mykyta Yaroshenko. 2025.
\newblock {GLiClass}: Generalist lightweight model for sequence classification tasks.
\newblock arXiv:2508.07662.

\bibitem[{Tobi-Bueck}(2025)]{ticketsdata}
{Tobi-Bueck}. 2025.
\newblock Customer support tickets.
\newblock Hugging Face dataset \texttt{Tobi-Bueck/customer-support-tickets} (CC BY-NC 4.0).

\bibitem[Vovk et~al.(2005)]{vovk2005}
Vladimir Vovk, Alex Gammerman, and Glenn Shafer. 2005.
\newblock \emph{Algorithmic Learning in a Random World}.
\newblock Springer.

\bibitem[Warner et~al.(2024)]{warner2024modernbert}
Benjamin Warner, Antoine Chaffin, Benjamin Clavi\'e, Orion Weller, Oskar Hallstr\"om, Said Taghadouini, Alexis Gallagher, Raja Biswas, Faisal Ladhak, Tom Aarsen, Nathan Cooper, Griffin Adams, Jeremy Howard, and Iacopo Poli. 2024.
\newblock Smarter, better, faster, longer: A modern bidirectional encoder for fast, memory efficient, and long context finetuning and inference.
\newblock arXiv:2412.13663.

\bibitem[Warstadt et~al.(2019)]{warstadt2019cola}
Alex Warstadt, Amanpreet Singh, and Samuel~R. Bowman. 2019.
\newblock Neural network acceptability judgments.
\newblock \emph{Transactions of the Association for Computational Linguistics}, 7:625--641.

\bibitem[Williams et~al.(2018)]{williams2018mnli}
Adina Williams, Nikita Nangia, and Samuel Bowman. 2018.
\newblock A broad-coverage challenge corpus for sentence understanding through inference.
\newblock In \emph{Proceedings of NAACL-HLT 2018}, pages 1112--1122.

\bibitem[Williams and Young(2007)]{williams2007pomdp}
Jason~D. Williams and Steve Young. 2007.
\newblock Partially observable {M}arkov decision processes for spoken dialog systems.
\newblock \emph{Computer Speech \& Language}, 21(2):393--422.

\bibitem[Williams(1992)]{williams1992}
Ronald~J. Williams. 1992.
\newblock Simple statistical gradient-following algorithms for connectionist reinforcement learning.
\newblock \emph{Machine Learning}, 8:229--256.

\bibitem[Yang et~al.(2025)]{qwen3}
An~Yang, Anfeng Li, Baosong Yang, Beichen Zhang, Binyuan Hui, Bo~Zheng, Bowen Yu, Chang Gao, Chengen Huang, Chenxu Lv, et~al. 2025.
\newblock Qwen3 technical report.
\newblock arXiv:2505.09388.

\bibitem[Y{\i}lmaz et~al.(2026a)]{yilmaz2026moganbert}
Furkan Y{\i}lmaz, Habibe Aleyna Ta\c{s}demir, and Muhammed Faruk G\"{o}zay. 2026a.
\newblock {MoganBert-TR}: A {T}urkish encoder foundation model trained from scratch with a {CLM}-to-{MLM} curriculum.
\newblock arXiv:2608.25768.

\bibitem[Yu et~al.(2020)]{yu2020interactive}
Lili Yu, Howard Chen, Sida~I. Wang, Tao Lei, and Yoav Artzi. 2020.
\newblock Interactive classification by asking informative questions.
\newblock In \emph{Proceedings of ACL 2020}, pages 2664--2680.

\bibitem[Zaratiana et~al.(2024)]{zaratiana2024gliner}
Urchade Zaratiana, Nadi Tomeh, Pierre Holat, and Thierry Charnois. 2024.
\newblock {GLiNER}: Generalist model for named entity recognition using bidirectional transformer.
\newblock In \emph{Proceedings of NAACL 2024}, pages 5364--5376.

\bibitem[Zhang and Choi(2023)]{zhang2023clarify}
Michael J.~Q. Zhang and Eunsol Choi. 2023.
\newblock Clarify when necessary: Resolving ambiguity through interaction with {LMs}.
\newblock arXiv:2311.09469.

\bibitem[Zhang et~al.(2015)]{zhang2015charcnn}
Xiang Zhang, Junbo Zhao, and Yann LeCun. 2015.
\newblock Character-level convolutional networks for text classification.
\newblock In \emph{Advances in Neural Information Processing Systems 28}, pages 649--657.

\bibitem[Zhang et~al.(2019)]{zhang2019paws}
Yuan Zhang, Jason Baldridge, and Luheng He. 2019.
\newblock {PAWS}: Paraphrase adversaries from word scrambling.
\newblock In \emph{Proceedings of NAACL-HLT 2019}, pages 1298--1308.

\end{thebibliography}
\end{document}